%% file: main-arxiv.tex
\documentclass[11pt]{article}

\RequirePackage{amsthm,amsmath,amsfonts,amssymb}
\RequirePackage[authoryear,round]{natbib}

\RequirePackage{graphicx}

\usepackage{fullpage}
\usepackage[utf8]{inputenc}
\usepackage[T1]{fontenc}
\usepackage{microtype}
\usepackage{csquotes}
\usepackage[osf,sc]{mathpazo}
\RequirePackage[scaled=0.90]{helvet}
\RequirePackage[scaled=0.85]{beramono}
\RequirePackage{textcomp}

\usepackage[skip=2mm,indent=5mm]{parskip}

\theoremstyle{plain}
\newtheorem{theorem}{Theorem}

\theoremstyle{remark}
\newtheorem{definition}[theorem]{Definition}

\input{macros}

\usepackage[small]{caption}
\usepackage[colorlinks=true,linkcolor=DarkBlue,citecolor=DarkGreen,urlcolor=DarkBlue]{hyperref}
\hypersetup{
    unicode=false,          
    pdftoolbar=true,        
    pdfmenubar=true,        
    pdffitwindow=false,      
    pdfnewwindow=true,      
    filecolor=DarkRed,      
    pdftitle={Performative Prediction: Past and Future},
    pdfauthor={Moritz Hardt, Celestine Mendler-Dünner},
}

\usepackage{tikz}
\usetikzlibrary{shapes,arrows}

\title{Performative Prediction: Past and Future}

\makeatletter
\newcommand{\printfnsymbol}[1]{%
  \textsuperscript{\@fnsymbol{#1}}%
}
\makeatother

\usepackage{authblk}
\stepcounter{footnote}
\addtocounter{footnote}{-1}
\author[1]{Moritz Hardt}
\author[1,2]{Celestine Mendler-Dünner}
\affil[1]{Max-Planck Institute for Intelligent Systems, Tübingen and Tübingen AI Center}
\affil[2]{ELLIS Institute, Tübingen}
\date{}                     
\begin{document}

\maketitle

\begin{abstract}
Predictions in the social world generally influence the target of prediction, a phenomenon known as performativity. Self-fulfilling and self-negating predictions are examples of performativity. Of fundamental importance to economics, finance, and the social sciences, the notion has been absent from the development of machine learning that builds on the static perspective of pattern recognition. In machine learning applications, however, performativity often surfaces as distribution shift. A predictive model deployed on a digital platform, for example, influences behavior and thereby changes the data-generating distribution. We discuss the recently founded area of performative prediction that provides a definition and conceptual framework to study performativity in machine learning. A key element of performative prediction is a natural equilibrium notion that gives rise to new optimization challenges. What emerges is a distinction between learning and steering, two mechanisms at play in performative prediction. Steering is in turn intimately related to questions of power in digital markets. The notion of performative power that we review gives an answer to the question how much a platform can steer participants through its predictions. We end on a discussion of future directions, such as the role that performativity plays in contesting algorithmic systems.
\end{abstract}

\input{content}

\bibliographystyle{plainnat} 
\bibliography{references}

\end{document}

%% file: macros.tex
\usepackage{xcolor}

\definecolor{DarkGreen}{rgb}{0.075,0.375,0.075}
\definecolor{DarkRed}{rgb}{0.5,0.1,0.1}
\definecolor{DarkBlue}{rgb}{0.1,0.1,0.5}
\definecolor{Gray}{rgb}{0.2,0.2,0.2}

\newcommand{\thetaPS}{{\theta_{\mathrm{PS}}}}
\newcommand{\thetaPO}{{\theta_{\mathrm{PO}}}}

\newcommand{\R}{\mathbb{R}}
\newcommand{\risk}{\mathrm{Risk}}
\newcommand{\PR}{\mathrm{PR}}
\newcommand{\cW}{\mathcal{W}}
\newcommand{\cA}{\mathcal{A}}
\newcommand{\cP}{\mathcal{P}}
\newcommand{\cD}{\mathcal{D}}
\newcommand{\cX}{\mathcal{X}}
\newcommand{\cY}{\mathcal{Y}}
\newcommand{\cU}{\mathcal{U}}
\newcommand{\cF}{\mathcal{F}}

\newcommand{\distance}{\mathrm{dist}}
\newcommand{\PP}{\mathrm{P}}

\DeclareMathOperator*{\argmin}{arg\,min}
\DeclareMathOperator*{\argmax}{arg\,max}

\DeclareMathOperator*{\E}{\mathbb{E}}

%% file: content.tex
\section{Introduction}

\input{introduction}

\section{Motivation: the GMS theorem}
\label{sec:GMStheorem}

\cite{grunberg54} distinguished between \emph{private} and \emph{public} predictions. Private predictions have no causal powers, whereas public predictions can alter the course of events. They raised a fundamental question: Under what conditions will a public prediction, even though it may change the turn of events, still come true?

To study this question formally, assume we aim to predict an outcome $y$. Assume $y$ is bounded, and without loss of generality, let $y\in[0,1]$. Denote the prediction $\hat y$. A perfect prediction corresponds to the case that $\hat y = y$. We express the relationship between the prediction and the outcome it causes through the response function
\begin{equation}
y = R(\hat y)\,.
\end{equation}
Thus, the question of whether public prediction is possible amounts to asking whether there exists a prediction for which $\hat y = R(\hat y)$. Grunberg, Modigliani, and Simon were the first to provide a positive answer. They identified continuity of $R$ as a sufficient condition for the feasibility of correct public prediction under performativity. Their result follows from  Brouwer's fixed point theorem. The one-dimensional case relevant here is just the intermediate value theorem from calculus. Let us illustrate the key argument in Figure~\ref{fig:GMS}.

\begin{figure}[h!]
\begin{center}
\begin{tikzpicture}
\fill [black!5!white] (0,0) rectangle (3.6,3.6);
\draw[thick,->] (0,0) -- (4,0) node[anchor=north west] {predicted value $\hat y$};
\draw[ thick,->] (0,0) -- (0,4) node[anchor=south ] {outcome $y$};
\draw[thick,dashed] (0,0) -- (3.7,3.7) node[anchor=south ] {perfect prediction};
\draw[thick,dotted] (1.95,0) -- (1.95,1.95);
\draw[thick,dotted] (0,1.95) -- (1.95,1.95);
\draw (1.95,0) node [anchor=north]{$y^*$};
\draw (0,1.95) node [anchor=east]{$y^*$};
\draw (0,1.2) node [anchor=east]{$R(0)$};
\draw (3.6,2.8) node [anchor=west]{$R(1)$};
\draw (3.6,0.1) -- (3.6,-0.1) node [anchor=north]{$1$};
\draw (0.1,3.6) -- (-0.1,3.6) node [anchor=east]{$1$};
\draw[ultra thick] (0,1.2) .. controls (1,2.6) and (2.5,1.3) .. (3.6,2.8);
\end{tikzpicture}
\end{center}
\caption{Simon's argument for the existence of stable points.}
\label{fig:GMS}
\end{figure}

First, draw a point anywhere on the $y$-axis, representing the realized outcome $R(0)$ when the prediction is $\hat y=0$. Then, mark a second point on the vertical line $\hat y=1$, representing the realized outcome $R(1)$. Under the constraint that your pencil may not leave the square it is impossible to connect the two points without touching the dashed line or lifting the pencil. Thus, for any continuous relationship between $y$ and $\hat y$ there must be at least one point~$y^*$ for which~$y^* = R(y^*)$. Thus, $y^*$ comes true \emph{after} being published.

\section{Performative prediction}
\label{sec:performative-prediction}

\cite{PZMH20} proposed a formalism to capture performativity in the context of supervised machine learning. Compared to the economic setting from the previous section, in machine learning predictions take instance-level features~$x$ into account and come from a predictive model~$f_\theta$ parameterized by a vector $\theta\in\Theta$. The predictive model takes a feature vector $x\in \cX$ as input and maps it to a prediction $\hat y=f_\theta(x)\in\cY$. Let $\cX\subseteq \mathbb R^d$ denote the feature space and $\cY\subseteq \mathbb R$ denote the output space. We assume model parameters are chosen from a closed and convex parameter space $\Theta\subseteq \R^d$.
A predictive model is generally deployed across a population to simultaneously make predictions for multiple individuals. The setting applies to both regression and classification problems.

Following the standard risk formulation of supervised learning, we assess the fit of a model by its \emph{loss} on a distribution. We denote by $\Delta(\cX\times\cY)$ the simplex of probability distributions over the domain $\cX\times\cY.$  For a given distribution $\cD\in \Delta(\cX\times\cY)$ the risk of a model~$\theta$ for a loss function $\ell$ is given as
\begin{equation}
\risk(\theta, \cD) = \E_{z\sim \cD}\left[\ell(\theta;z)\right]\,.
\label{eq:risk}
\end{equation}
What distinguishes performative prediction from standard supervised learning is that the data-generating distribution~$\cD$ is not fixed and external to the prediction problem. Instead, the data-generating distribution may change depending on the predictive model. To this end, the formalism of performative prediction introduces a so-called \emph{distribution map} as a key conceptual device.

\subsection{Distribution map}

The distribution map expresses the dependence of the data-generating distribution on the predictive model. 
It is a mapping
\[
\cD\colon\Theta\to\Delta(\cX\times\cY)\,.
\]
For every parameter vector $\theta\in\Theta$, the distribution ~$\cD(\theta)\in\Delta(\cX\times\cY)$ describes the data-generating distribution over data instances $z=(x,y)$ that results from deploying the predictive model~$f_\theta$.

The distribution map gives a general way to describe distribution shifts in response to model deployment. The formal setup is abstract in how it does not posit any specific mechanism for the distribution shift. But it makes explicit how performativity enters the statistical formalism of machine learning. 

In practice, a data-generating distribution might change over time for many different reasons. Performative prediction focuses on one and only one mechanism: the effect of model deployment on the data. The simplicity of this formalism is its virtue. In particular, the distribution map is stateless, meaning that deploying the same model at any point in time repeatedly leads to the same distribution. However, researchers have proposed many extensions to the basic formalism that we discuss throughout. 

\paragraph*{Sensitivity.} We typically work with some regularity assumption on the distribution shift to quantify how far the distribution map is from constant. While there are different ways to quantify this distance, we will mostly rely on a definition of sensitivity proposed by \cite{PZMH20}. It amounts to a Lipschitz condition with respect to the $\ell_2$-norm over the parameter space and the Wasserstein distance over the probability simplex.
\begin{definition}[Sensitivity] We say the distribution map $\cD(\cdot)$ is \emph{$\epsilon$-sensitive} if for all $\theta,\theta'\in\Theta$ it holds that
\[\cW(\cD(\theta),\cD(\theta'))\leq \epsilon \|\theta-\theta'\|_2\,,\]
where $\cW$ denotes the Wasserstein-1 distance.
\end{definition}
The special case where $\epsilon=0$ implies $\cD(\theta)=\cD(\theta') $ for all $\theta,\theta'\in\Theta$, recovering a static, non-performative setting.

\subsection{Performative stability}

Given the concept of a distribution map, it is natural to assess a model's risk with respect to the distribution that manifests from its deployment. This leads to two different solution concepts.

The first is a natural equilibrium notion, termed performative stability. Performative stability requires that the model looks optimal on the distribution it entails. More formally, a model $\thetaPS$ is called \emph{performatively stable} if it satisfies the fixed point condition
\begin{equation}
\thetaPS \in \argmin_{\theta\in \Theta} \;\risk(\theta, \cD(\thetaPS))\,.
\label{eq:PS}
\end{equation}
This means, based on data collected after the deployment of $\thetaPS$, there is no incentive to deviate from the model. It is optimal on the static problem defined by $\cD(\thetaPS)$. This offers a simple empirical check for certifying stability. Collect data in current conditions, solve a risk minimization problem on the data, and check if the model is at least as good as the risk minimizer.  In other words, the data we see at a stable point do not refute the optimality of the model. Echo chambers are therefore an apt metaphor for stable points: what we hear doesn't challenge our beliefs. That is not to say that there is no model of smaller loss.

\subsection{Performative optimality}
An alternative solution concept asks for the smallest possible risk \emph{post deployment} globally among all models.
More specifically, we say that a predictive model with parameters~$\thetaPO$ is \emph{performatively optimal} if it satisfies
\begin{equation}
\thetaPO \in \argmin_{\theta\in \Theta} \;\risk(\theta,\cD(\theta))\,.
\label{eq:PO}
\end{equation}
Performative optimality can be thought of as a Stackelberg equilibrium in a sequential game between the decision-maker deploying a model, and the population responding. 
We call the risk of a model on the distribution it entails the \emph{performative risk}, defined as
\begin{equation}\PR(\theta):=\risk(\theta,\cD(\theta))\,.
\label{eq:PR}
\end{equation}
Performatively optimal models minimize performative risk by definition and it always holds that $\PR(\thetaPO)\leq \PR(\thetaPS)$ for any performative optimum~$\thetaPO$ and stable point $\thetaPS.$ In a static setting the two solution concepts of stability and optimality coincide with the classical supervised learning solution.
In general, however, performatively stable points need not be optimal and optimal points need not be stable. In contrast to stability, certifying optimality is harder, as the data available to us, in general, do not tell us anything about the performance of any other model post deployment.~\looseness=-1

\subsection{Rewriting the rules of prediction}

Performative prediction rewrites the rules of prediction insofar as there are now two ways to be good at prediction. 

To see this, fix a model $\phi$ that we imagine provides the current data-generating distribution, and consider deploying another model~$\theta.$ Observe that the model~$\theta$ shows up in two places in the definition of the performative risk: in the first argument, reflecting the dependence of the loss~$\ell(\theta;\,z)$ on $\theta$, and in the second argument representing the dependence of the distribution~$\cD(\theta)$ on $\theta$. Thus, we can decompose the performative risk $\PR(\theta)$ as: 
\begin{align}
\PR(\theta)&=\risk(\theta,\cD(\phi))+ \left(\risk(\theta,\cD(\theta))- \risk(\theta,\cD(\phi))\right)
\label{eq:learning-vs-steering}
\end{align}
This tautology exposes two ways to achieve small performative risk. One is to optimize well in current conditions, that is, to minimize $\risk(\theta,\cD(\phi))$. The other is to \emph{steer} the data to a new distribution~$\cD(\theta)$ that permits smaller risk. The first mechanism is well known. It's what we conventionally call \emph{learning}. Learning is the classical way that, for example, a platform can discover and target consumer preferences. Steering is unique to performative prediction allowing the platform to push consumption towards a distribution more favorable to its objectives. Machine learning had exclusively studied the former, whereas the latter mechanism, as we will see, relates to important questions about the impact of predictive models.

\subsection{Revisiting economic forecasting}
Having outlined the framework, it's worth delineating performative prediction from its $20$th century ancestry in economics. There are at least three important differences. 
\begin{enumerate}
\item  In machine learning predictions come from \emph{parametric models}. The function class is not necessarily fully expressive; we do not presuppose the possibility of perfect private prediction. This implies the distinction between stability and optimality, because stable points are no longer necessarily simultaneously optimal.
\item In machine learning, individuals are described by \emph{features and labels}. Performativity can surface in both. Earlier it was only the outcome. This increases the expressivity of the performative prediction framework, allowing for different sources of performativity.
\item In macroeconomics it's a central forecast about the aggregate economy that has the causal powers to change the course of events. In performative prediction, it's the \emph{individual predictions} output by a predictive model that have the causal powers to change individual behavior and outcomes. 
In aggregate, this causes a change in response to the predictive model.~\looseness=-1
\end{enumerate}
A consequence is that in the GMS theorem the regularity assumptions on the response function directly concern the scalar aggregate outcome variable. In machine learning we work with regularity conditions concerning the population to which the predictive model is applied. Embedding the sociological concept of performativity into the risk minimization formulation of machine learning---and working out the consequences for standard learning practices---is a central ambition of performative prediction.

\section{Retraining under performativity}
\label{sec:retraining}

Using the standard language of machine learning, performativity surfaces in practice as what is commonly called distribution shift. Repeated retraining is the true and tried heuristic for dealing with distribution shift of all sorts.
This section studies the dynamics of retraining in performative prediction, exposing performatively stable points as natural fixed points of retraining. We present a constructive proof for the existence of performatively stable points and provide conditions for the convergence of retraining, together with several natural modifications to the algorithm and practical considerations.

\subsection{Repeated risk minimization}

Consider a conceptual algorithm formalizing the idea of model retraining. In every iteration the algorithm finds the minimizer on the distribution surfacing from the previous deployment. We term this strategy repeated risk minimization (RRM). More formally, for an arbitrary initialization $\theta_0\in\Theta$, RRM defines an iterate sequence as follows:
\begin{equation}\theta_{k+1}:=\argmin_{\theta\in \Theta} \;\risk(\theta,\cD(\theta_k))\,.
\label{eq:RRM}
\end{equation}
We resolve the issue of the minimizer not being unique by setting $\theta_{k+1}$ to an arbitrary point in the argmin set.
It is not hard to see that performatively stable points are fixed points of RRM. A fundamental theorem in performative prediction provides conditions for the existence of stable points and establishes that RRM converges to a stable point at a linear rate.

The result relies on the strong convexity of the loss function with respect to changes in the parameter vector, and a corresponding smoothness condition with respect to changes in the data. More precisely, we say the loss $\ell(\theta;z)$ is \emph{$\beta$-smooth} in $z$ if
\begin{align*}
\|\nabla \ell(\theta;z) - \nabla \ell(\theta;z')\|&\leq\ \beta\|z-z'\|
\end{align*}
 for all $\theta\in\Theta$ and $z,z'\in\mathcal Z$ with $\mathcal Z:=\cup_{\theta\in\Theta}\mathrm{supp}(\cD(\theta)).$

Throughout, we use the notation $\nabla\ell$ to denote the gradient of the loss with respect to the model parameters $\theta$ and $\|\cdot\|$ denotes the $\ell_2$-norm.

\begin{theorem}[\cite{PZMH20}] Suppose that the loss function $\ell(\theta,z)$ is $\gamma$-strongly convex in $\theta$ and $\beta$-smooth in $z$. Then, repeated retraining defined in Equation~\ref{eq:RRM} converges to a unique stable point as long as the sensitivity of the distribution map $\cD(\cdot)$ satisfies $\epsilon< \frac \gamma \beta$. Furthermore, the rate of convergence is linear, and for $k\geq 1$ the iterates satisfy 
\[\|\theta_k-\thetaPS\|\leq \left(\frac{\epsilon \beta} \gamma\right)^k \|\theta_0-\thetaPS\|\,.\]
\label{thm:RRM}
\end{theorem}

\begin{proof}
The key step in the proof is to show that the mapping 
$G(\theta):=\argmin_{\theta'\in\Theta} \risk(\theta',\cD(\theta))$ satisfies the following inequality
\begin{equation}
\|G(\theta)-G(\phi)\|\leq \frac{\epsilon\beta}{\gamma}\|\theta-\phi\|\,.
\label{eq:contraction}
\end{equation}
Then, for $\epsilon< \frac \gamma \beta$ it follows that $G$ is a contractive map. This implies the existence of a unique stable point by Banach's fixed point theorem. Furthermore, the linear convergence rate follows by replacing $\phi$ with $\thetaPS$ and applying the bound recursively.

It remains to derive Inequality~\ref{eq:contraction} from the assumptions of the theorem. Therefore, consider the static optimization problem induced by the deployment of a model $\theta$ with minimizer at $G(\theta)$. From strong convexity it follows that for any  $\theta,\theta'\in\Theta$
\[\E_{z\sim\cD(\theta)} \nabla \ell(\theta';z)^\top (\theta'-G(\theta)) \geq\gamma \|\theta'-G(\theta)\|^2\,.\]
To obtain a corresponding upper bound on the left-hand-side we note that $\nabla \ell(\theta';z)^\top v$ is a $\beta\|v\|$-Lipschitz function in $z$ for any $v$ and we can apply the Kantorovich-Rubinstein duality to relate the expected function value across $\cD(\theta)$ and $\cD(\phi)$.

More, formally, for any $v$ and $\theta, \theta'\in\Theta$ it holds that 
\begin{align}
& v^\top\Big(\E_{z\sim\cD(\theta)} \nabla \ell(\theta';z)-\E_{z\sim\cD(\phi)} \nabla \ell(\theta';z)\Big) \notag \\
& \leq \beta \|v\| \cdot \cW(\cD(\theta),\cD(\phi)) \label{eq:shiftbound}\\
&\leq \epsilon\beta\|v\|\cdot \|\theta-\phi\|\,. \notag
\end{align}
Finally, we instantiate the first bound with $\theta'=G(\phi)$, such that $\mathbb E_{z\sim\cD(\phi)} \nabla \ell(\theta';z)=0$ and choose $v = G(\phi)-G(\theta)$. Then, comparing the two bounds completes the proof.
\end{proof}




\citet{PZMH20} showed that with any of the three assumptions removed RRM is no longer guaranteed to converge. In particular, convexity is not sufficient for finding stable points. The reason is that the per-step reduction in risk achieved on any fixed distribution needs to overcome the effect of the distribution shift caused by deploying the model update. The latter can potentially induce an error that is quadratic in the magnitude of the parameter change. 

The result in Theorem~\ref{thm:RRM} can be seen as an algorithmic analog to the GMS Theorem in the context of risk minimization. It gives a condition for when a model remains optimal even under the distribution shift that surfaces from its deployment. Notably this result does not require that that problem is realizable, nor that perfect private prediction is possible. Instead, it focuses on local optimality with respect to a risk function alone.

\paragraph*{Stability and retraining in more general settings.}
The concept of performative stability as a fixed point of retraining has been generalized to several extensions of the performative prediction framework.  
\cite{wang2023constrained} study equilibrium points in the presence of model-dependent optimization constraints. 
\cite{li22multiplayer} considered a collaborative learning setting where multiple decision makers optimize simultaneously while each agent only observes and influences parts of the distribution. Here the existence of stable points depends on the average sensitivity across the local distributions. \cite{narang23multiplayer} proposed a game-theoretic setting where the population data reacts to competing decision makers’ actions. They show that under appropriate conditions on the loss function simultaneous retraining converges to stable points in this multi-player setting.  Similarly, in \cite{piliouras23multiAgent} the learning problems of multiple agents are coupled through the data distribution that depends on the actions of all agents. 
\cite{jin24federated} study convergence to performative stability in heterogeneous federated learning settings. \cite{wang23ppnetwork} extend the model to account for network effects and the locality of mutual influence. 
\cite{brown22stateful} formulated a time-dependent performative prediction problem where the decision-maker seeks to optimize the reward under the fixed point distribution induced by the deployment of their model. They prove convergence of retraining to stable points under a generalized notion of sensitivity. 
\cite{gois2024performative} account for changes in performative effects due to inaccurate predictions and a resulting loss in trust.
\cite{mandal2023performative} study performativity in reinforcement learning where the transition probabilities of the Markov decision process change with the deployed policy. Proposing a corresponding sensitivity assumption they prove convergence of repeated retraining to stable policies. 
Beyond supervised learning, \cite{bertrand2023stability} extend the concept of performative stability to generative models trained on their own data. They prove convergence of retraining as long as the amount of generated data from previous models entering future training data is small enough. 
What unifies all these results is the interpretation of model retraining as a natural equilibrium dynamic under performativity that converges under appropriate regularity assumptions on the performative shift.

\subsection{Gradient-based optimization}
Let us replace the optimization oracle in RRM with a gradient step. This defines the following repeated gradient descent (RGD) procedure
\[\theta_{k+1} := \theta_k-\eta \E_{z\sim\cD(\theta_k)}\nabla \ell(\theta_k;z)\,.\]
To study the convergence of RGD we make the additional assumption that the loss function is smooth with respect to changes in the model parameters. Then, from classical results in convex optimization we know that in the static setting, repeated gradient descent, with appropriately chosen step size, makes progress towards the minimizer in each step, converging as $k\rightarrow \infty$. What differentiates the performative prediction setting is that the progress in each step is made on a moving sequence of distributions determined by the trajectory of the algorithm. By carefully choosing the step size to control for the shift induced by the update, it can be shown that in the regime $\epsilon<\frac \gamma \beta$ RGD can achieve linear convergence to stability, similar to RRM, although at a slower rate.

The RGD algorithm was first studied by \cite{PZMH20}; \cite{mendler20stochasticPP} refined the analysis. We refer to the latter reference for a formal proof and an exact statement of the rate. Here we present an illustrative argument by \cite{drusvyatskiy23stochastic} for why we can expect RGD to converge despite distribution shift along the trajectory. To this end, we show that RGD can be seen as solving the equilibrium problem $\risk(\theta,\cD(\thetaPS))$ using biased gradients.
We use the shorthand notation \[ g_\mathrm{PS}(\theta)\,:=\,\E_{z\sim\cD(\thetaPS)}\nabla \ell(\theta;z)\] to denote the gradient at $\theta$ evaluated on the equilibrium distribution $\cD(\thetaPS)$. Similarly, we write $g_{k}(\theta)$ to denote the gradient on the problem induced at step~$k$ by $\cD(\theta_k)$.  We inspect how $g_k(\theta)$ relates to $g_{\mathrm{PS}}(\theta)$.
To do so, we recall the contraction argument from the previous section. Using Inequality~\ref{eq:shiftbound} with $v$ being the unit vector, we can bound the deviation of $g_k$ from $g_{\mathrm{PS}}$  as
\begin{align}\|g_k(\theta_k)-g_{\mathrm{PS}}(\theta_k)\|&\leq \epsilon\beta \|\theta_k-\thetaPS\|\,.
\label{eq:bias}
\end{align}

To interpret this bound, note that $\gamma$-strong convexity of the loss allows us to relate the parameter distance back to the gradient norm as $\|\theta_k-\thetaPS\|\leq\frac 1 \gamma  \|g_{\mathrm{PS}}(\theta_k)\|$. Combined with a geometric argument this implies
\[\cos\big(\angle (g_k(\theta_k), g_{\mathrm{PS}}(\theta_k))\big)\leq \sqrt{1-\left(\frac {\epsilon\beta}{\gamma}\right)^2}\,.\]

Hence, in the regime $\epsilon< \frac \gamma \beta$ the gradients computed on~$\cD(\theta_k)$ and~$\cD(\thetaPS)$ are aligned and span an angle strictly smaller than $90^{\circ}$. Thus, the bias caused by performativity along the trajectory is never making update steps point against the gradient flow on the equilibrium problem, achieving convergence eventually.  

For quantifying the rate of convergence it is important to note that according to Equation~\ref{eq:bias} the bias diminishes as the stable point is approached, and it does so linearly in parameter distance. This is the reason why the linear rate of gradient descent for smooth and strongly convex functions can be preserved under performativity, given the appropriate regularity condition on the shift. We refer to~\cite{mendler20stochasticPP} for the exact statement of the rate. Characterizing similar conditions for the convergence of retraining under weaker notions of strong convexity \cite[e.g.,][]{polyak1963GradientMF,necoara2015LinearCO} offers interesting open questions.

\subsection{Stochastic Optimization}

The algorithms we discussed so far have been largely conceptual since they assume knowledge of the data distribution for every update step. In this section we focus on retraining algorithms that can access the distribution only through samples. This adds another technical step to the convergence analysis since the stochastic variance propagates into the distribution.

\paragraph*{Stochastic gradient descent.}
We start by considering stochastic gradient descent (SGD) in performative settings. In each step a single sample is drawn from the distribution induced by the most recently deployed model. The sample is then used to update the model before performing the next deployment. Formally, we can write the update step as
\begin{equation}
\theta_{k+1}\,:=\,\theta_k-\eta_k\nabla \ell(\theta_k;z)\qquad\text{with } z\sim\cD(\theta_k)\,,
\end{equation}
where $\eta_k>0$ denotes the stepsize chosen at step $k$.

The central difference to stochastic optimization in a static setting is that the distribution from which samples are drawn depends on the trajectory of the learning algorithm. Thus, not only the deviation from the gradient due to stochasticity needs to be controlled, but also the effect it has on the induced distribution. \cite{mendler20stochasticPP} showed that the sublinear convergence of SGD can be preserved in the regime $\epsilon<\frac \gamma\beta$ with a decreasing stepsize schedule that accounts for the variance of the gradients and the size of the shift it induces.

As in the static setting, the result relies on a bounded variance assumption on the stochastic updates that is assumed to be preserved under distribution shift.
\cite{mendler20stochasticPP} use the following $(\sigma,L)$-boundedness assumption which is assumed to hold for all $\theta,\theta'\in\Theta$
\begin{equation}
\E_{z\sim\cD(\theta)}\|\nabla \ell(\theta';z)\|^2\leq \sigma^2 + L^2\cdot \|\theta'-G(\theta)\|^2\,.
\label{eq:sig}
\end{equation}
Given the additional assumption that the loss function is $\beta$-smooth in the model parameters, then, for $L=\beta$, the above bound is implied by the more classical $\sigma^2$ bounded variance assumption on the expected deviation of the stochastic gradients from the mean $\mathbb{E}_{z\sim\cD(\theta)} \nabla \ell (\theta';z)$. For $L=0$ it recovers the stronger textbook assumption of bounded gradient norm.

Based on this assumption Theorem~\ref{thm:SGD} proves that even  retraining methods that have access to only a single sample can converge to stability in the regime $\epsilon<\frac\gamma \beta$.
\begin{theorem}[\cite{mendler20stochasticPP}]
Suppose that the loss $\ell(\theta,z)$ is $\gamma$-strongly convex, $\beta$-smooth in $z$ and in $\theta$, and the variance of the gradient is $(\sigma,L)$-bounded. Assume the distribution map is $\epsilon$-sensitive with  $\epsilon<\frac\gamma\beta$. Then, for a stepsize sequence chosen as $\eta_k= \left((\gamma-\epsilon\beta)k+8L^2/(\gamma-\epsilon\beta)\right)^{-1} $ the iterates satisfy \[\E {\|\theta_{k+1}-\thetaPS\|^2}\leq \frac {M}{(\gamma-\epsilon \beta)^2 k + 8 \beta^2}\,,\]
where  $M:=\max\big(2\sigma^2,8L^2 \|\theta_1-\thetaPS\|^2\big)$.
\label{thm:SGD}
\end{theorem}

\begin{proof}[Proof Sketch]
We provide a modified version of the original proof and analyzes SGD as a method for solving the static risk minimization problem at equilibrium, building on the intuition of~\cite{drusvyatskiy23stochastic}. Since in every step a sample $z$ is drawn iid from $\cD(\theta_k)$ the SGD update step is unbiased with respect to $g_{k}(\theta_k)$. From the previous section we know that we can control the bias of $g_{k}$ with respect to $g_{\mathrm{PS}}$, and repeated gradient descent converges to the performative stability. Thus, it remains to show that the variance of the stochastic gradients decreases sufficiently quickly as we approach $\thetaPS$.

To this end, we use the $(\sigma,L)$-boundedness assumption together with the fact that $G$ is a contractive map with Lipschitz constant $\frac{\epsilon \beta} \gamma$ to get
\begin{align*}
\E_{z\sim\cD(\theta_k)}\|\nabla \ell(\theta_k;z)\|^2
&\leq \sigma^2 + L^2 \left(1+\frac {\epsilon\beta} \gamma\right)^2 \Delta_\mathrm{PS}^2(\theta_k)\,,
\end{align*}
where $\Delta_\mathrm{PS}^2(\theta_k):=\mathbb E \|\theta_{k}-\thetaPS\|^2$. 
Combining this bound with the analysis of repeated gradient descent yields the key recursion:
\begin{align}
\E {\|\theta_{k+1}-\thetaPS\|^2} &\leq \eta_k^2 \sigma^2 +  \delta_k\E {\|\theta_k-\thetaPS\|^2} \label{eq:recursion}
\end{align}
with $\delta_k = 1-2(\gamma-\epsilon\beta)\eta_k 
+\eta_k^2L^2\Big(1+\tfrac {\epsilon\beta} \gamma\Big)^2\,.$

The decreasing stepsize schedule proposed in Theorem~\ref{thm:SGD} defines a trade-off between the progress on the static problem and the bias induced with respect to the equilibrium problem.
Given the recursion, a simple inductive argument suffices to show the claimed bound. See \cite{mendler20stochasticPP} for technical details.
\end{proof}

It is illustrative to compare the bound in Theorem~\ref{thm:SGD} to classical results on SGD (see, e.g., \cite{bottou18opt}).
Therefore, let us focus on the key recursion in Equation~\ref{eq:recursion}. The recursion recovers classical results for SGD in the static case by setting $\epsilon=0$. 
Furthermore, for the special case of $L=0$ the recursion recovers those of classical SGD analyses in the static setting where the typical strong convexity parameter $\gamma$ is replaced by $\gamma-\epsilon\beta$. From this analogy it directly follows that sublinear convergence at the rate $O(\frac 1 k)$ for decreasing stepsize $\eta_k=\frac \eta k$ can be achieved as long as $\eta<(2(\gamma-\epsilon\beta))^{-1}$. The result in Theorem~\ref{thm:SGD} provides a stronger guarantee by more carefully trading-off between the progress on the static problem and the bias induced with respect to the equilibrium problem.

\paragraph*{Beyond SGD.}
The idea that stochastic algorithms in the presence of  performative distribution shift are implicitly solving the static problem $F(\theta):=\risk(\theta;\cD(\thetaPS))$ corrupted by a vanishing bias can be generalized beyond SGD. 
\cite{drusvyatskiy23stochastic} applied this principle to various popular online learning algorithm to translate their rate of convergence from the static setting to the dynamic setting. This includes proximal point methods, inexact optimization methods, accelerated variants, and clipped gradient methods. We refer to their paper for the formal statements. Providing a different perspective, \cite{cutler2023stochastic} study retraining through the lens of variational inequalities and prove asymptotic minimax optimality of projected gradient descent in the performative setting.\looseness=-1


\subsection{Practical considerations} 
So far we assumed that the updated model is deployed after every stochastic update, immediately causing a shift in the data distribution. We refer to this strategy as \emph{greedy deploy}.
However, in practice, model deployments may come with high costs. Thus, it might be beneficial to collect additional samples from any given distribution and further refine the model update before deployment. Adapting the terminology from \cite{grunberg54}, we call updates that are deployed \emph{public model updates}, and updates that are done offline \emph{private model updates}. This distinction between public and private updates adds a new dimension to stochastic learning, not typically found in static settings. With every sample that arrives the learner has to decide whether to deploy the updated model and trigger a distribution shift or continuing to collect more sample from the current distribution.

\begin{figure}[t!]
    \centering
    \includegraphics[width = 0.8\columnwidth]{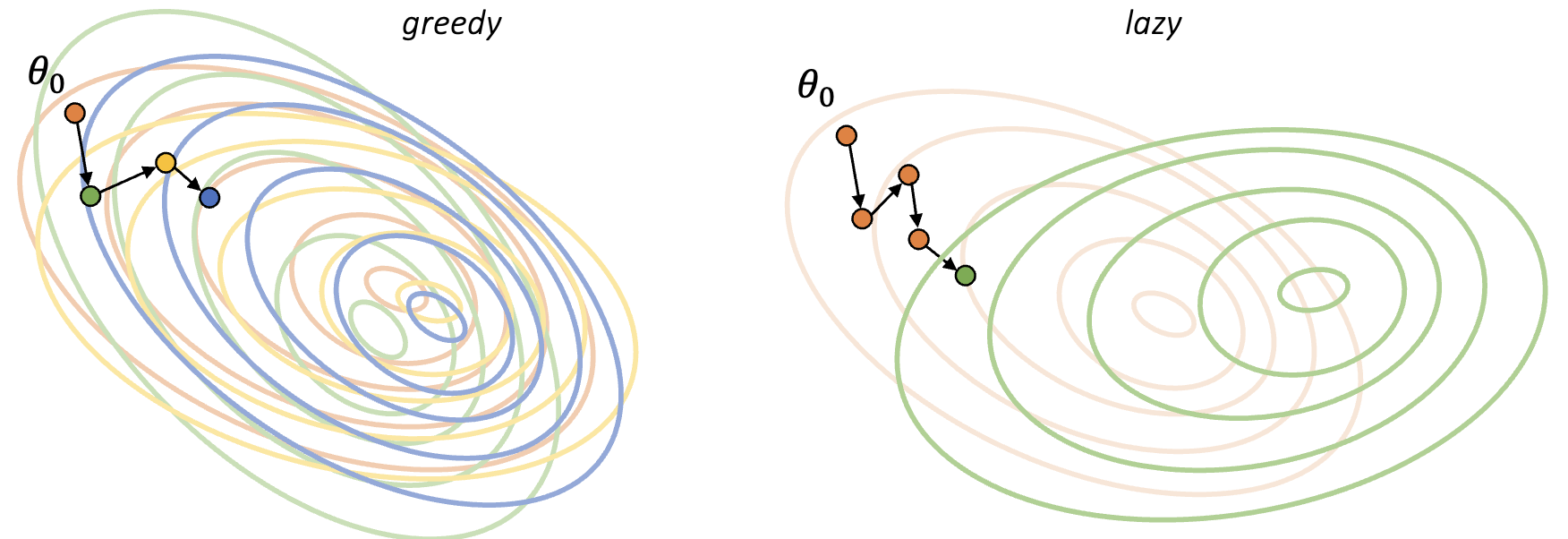}
    \caption{Different deployment strategies in stochastic optimization for performative prediction. Left: Greedy deploy publishes the model after every stochastic update. Right: Lazy deploy processes multiple samples offline before releasing the updated model. }
    \label{fig:SGD}
\end{figure}

To illustrate this trade-off, let us consider a natural variant of SGD that performs $n(k)$ stochastic update steps in between deployment $k$ and $k+1$ based on repeated sampling from $\cD(\theta_k)$. We call such a strategy \emph{lazy deploy}.
\cite{mendler20stochasticPP} studied the convergence properties of lazy deploy for the case where 
\[
n(k)\propto k^\alpha\qquad\text{for }\alpha >0\,.
\] 
In their algorithm, the optimization problem in each step $k\geq 1$ is treated as an independent, static optimization problem and SGD with a decaying stepsize schedule is used to solve it.  Thus, conceptually, the lazy deploy algorithm can be viewed as an approximation to the RRM procedure, where the approximation error decreases as $k$ grows.

Bounding the approximation error and analyzing the corresponding effect on the distribution shift, \cite{mendler20stochasticPP} gave a convergence guarantee for lazy deploy under the same conditions as greedy deploy in Theorem~\ref{thm:SGD}. Their result shows that lazy deploy can reduce the number of deployments at a cost of increased sample complexity. More formally, to achieve a suboptimality $\mathrm E [\|\theta_k-\thetaPS\|^2]\leq \delta$ the number of deployments scale as 
\[\mathrm{greedy:}\; O\big(1/\delta\big)\rightarrow \mathrm{lazy}:\; O\big(1/\delta^\alpha\big)\]
and the corresponding sample complexity scales as 
\[\mathrm{greedy:} \;O\big(1/\delta\big)\rightarrow \mathrm{lazy}:\;  O\big(1/\delta^{\frac{1+\alpha}\alpha}\big)\,.\]
Comparing the asymptotic properties of greedy and lazy deploy illustrates that depending on the cost of sample collection and the cost of model deployment, either of the two strategies can be more desirable. \cite{mendler20stochasticPP} further found that greedy deploy is particularly effective for weak performative shifts as it behaves like SGD in the static setting, whereas lazy deploy converges faster for larger $\epsilon$ by closely mimicking the behavior of RRM.

Beyond reducing for overheads of model deployments when searching for performatively stable points, efficiently reusing data from previous deployments to accelerate convergence was shown to be a promising avenue~\cite{khorsandi2024tight}. Further, more computationally efficient approaches to adjust complex models to the gradually shifting distribution could be interesting to explore in the context of performative prediction. These include practical approaches to transfer learning and domain adaptation in deep learning~\citep{long15transferCV, howard18universal} as well as recently popularized parameter-efficient fine-tuning techniques~\citep{houlsby19finetuning}. 

\subsection{Data feedback loops and dynamic benchmarks}

An interesting special case of performativity has been studied under the umbrella of \emph{data feedback loops}. A data feedback loop arises when models generate outputs that make it into future training datasets, which then in turn influence future models. From our discussion, we readily recognize the problem as a variant of the retraining dynamic. There are several different stories in the literature about what data feedback loops might lead to. Some worry about "echo chambers" or "filter bubbles" in settings like digital platforms and social networks \citep{jiang2019degenerate}. Others have pointed out concerns relating to algorithmic fairness \citep{ensign2018runaway, liu2018delayed}. In a simple model of bias amplification, \cite{taori2023data} identify consistent calibration as a sufficient condition for stability under such data feedback loops.  In the context of large language models, a particular worry is the hypothetical scenario of \emph{model collapse} \citep{shumailov2024ai, gerstgrasser2024model} in which a model fails to capture salient parts of the original distribution.

\emph{Dynamic benchmarks}, on the other hand, try to make a virtue out of data feedback loops. A dynamic benchmark evolves a test set over time by adding failure cases of current models to the test set \citep{dynabench}. The hope is that dynamic benchmarks enable steady progress as models continue to improve by learning from their own mistakes \citep{shirali2023theory}.

Performativity provides an overarching conceptual framework that informs both data feedback loops and dynamic benchmarks.

\section{Optimizing the performative risk}
\label{sec:po}
\label{sec:optimizing-performative-risk}

Beyond ad hoc risk minimization, finding models of small risk post deployment is an inherently different problem. Finding models that have small performative risk, necessitates to take the impact of models on the distribution into account.
From an optimization point of view, the main challenge of incorporating this additional dimension is that the induced distribution $\cD(\theta)$ can only be observed after the model $\theta$ is deployed.

In this section we discuss algorithmic approaches for finding performative optima that deal with this uncertainty by relying on sensitivity as a regularity condition on the distribution shift. We do not yet make any problem-specific assumption on the structure of $\cD(\theta)$.

\subsection{Derivative-free optimization}

The first approach we discuss applies gradient-based optimization to the performative risk directly. As gradients  are infeasible to compute without knowledge of the distribution map, we resort to derivative-free methods. Such methods explore modifications of the current iterate to investigate how they impact the performative risk. Such zero-order local information is then used to determine directions of improvement from finite differences, to optimize the performative risk. Variations of this approach were explored by~\cite{izzo21gd, miller2021outside, izzo22gradualPGD, Ray2022DecisionDependentRM} and \cite{cyffers24}. 

Extending this approach to a stateful setting, \cite{Ray2022DecisionDependentRM} study performative optimization in geometrically decaying environments assuming zeroth order or  first order oracle access to the performative risk. Adding optimization constraints to the problem, \cite{yan23POconstraint} take a primal-dual perspective and study performative optimization with approximate gradients.

While natural, a strong requirement for gradient-based approaches to efficiently minimize the performative risk from local information is that $\PR(\theta)$ satisfy convexity. Investigating this assumption \cite{miller2021outside} provided sufficient conditions on the structure of the distribution map for which (strong) convexity of the static problem is preserved under performativity. In particular, they showed that if the loss is $\gamma$-strongly convex and the distribution satisfies the following stochastic dominance condition 
\[\mathop{\E}_{z\sim\cD(\alpha \theta_1 + (1-\alpha)\theta_2)} \ell(\theta;z)\leq \mathop{\E}_{z\sim\alpha \cD( \theta_1) + (1-\alpha)\cD(\theta_2)} \ell(\theta;z),\]
for any $\theta,\theta_1,\theta_2\in\Theta$ and $\alpha\in(0,1)$, then $\PR(\theta)$ is $\gamma'=\beta-2\epsilon\gamma$ strongly convex. The condition corresponds to convexity of the risk in the distribution argument. Similar conditions have been extensively studied within the literature on stochastic orders, see~\cite{shaked2007stochastic}. A stronger result holds for certain families of distributions, such as location-scale families.

In general, however, the performative risk might not satisfy any structural properties which would imply that its stationary points have low performative risk. 
In fact, the performative risk can be non-convex, even for strongly convex losses and simple distribution shifts, as noted in~\cite{PZMH20}. Thus, controlling the loss alone is not sufficient to induce a \emph{nice} learning problem.  
Hence, relying on zero-order gradient methods can help achieve local improvements, but it may lead to suboptimal solutions in many natural cases.

\subsection{Bandit optimization}
\label{sec:bandit}

When $\PR(\theta)$ is non-convex, global exploration is necessary for finding performative optima. To this end, \cite{jagadeesan22regret} explored how tools from multi-armed bandits can be applied to guide exploration in performative prediction and find models of low performative risk. 

To set up the objective formally, assume at every step a model~$\theta_t$ is deployed, and after deployment the distribution $\cD(\theta_t)$ can be observed. Given the online nature of this task,  we measure the quality of a sequence $\theta_0,\theta_1,...,\theta_T$ by evaluating the performative regret where  performative optimality represents the target:
\[
\mathrm{Reg}(T):= \sum_{t\leq T} \PR(\theta_t)-\PR(\thetaPO).
\]

A unique characteristic of the performative prediction problem is the feedback structure. After each model deployment we get to observe the induced distribution $\cD(\theta_t)$, rather than only bandit feedback about the reward $\PR(\theta_t)$. We call this performative feedback. Together with knowledge of the loss function  this additional information can inform the reward of unexplored arms and allows for a tighter construction of confidence bounds. In the following we provide the main intuition for why this leads to an algorithm with a regret bound that primarily scales with the complexity of the distribution shift, rather than the complexity of the reward function.

\paragraph*{Performative confidence bounds.}  Access to the distribution $\cD(\theta_t)$ allows the learner to evaluate  $\risk(\theta';\cD(\theta_t))$ for any $\theta'\in\Theta$. Hence, for constructing confidence bounds on $\PR(\theta')$ we can extrapolate from $\risk(\theta';\cD(\theta_t))$, rather than extrapolating from $\PR(\theta_t)$.  
As a result, it remains to deal with uncertainty due to distribution shift alone, rather than the uncertainty in the reward function. Assuming $L_z$-Lipschitzness of the loss in $z$ and $\epsilon$-sensitivity of the distribution map we get the following upper-confidence bound on an unexplored model $\theta'\in\Theta$
\begin{align}
\PR(\theta')&\leq \min_t \;\risk(\theta';\cD(\theta_t))+L_z\epsilon\|\theta'-\theta_t\|\,.
\label{eq:ub}
\end{align}
The bound follows from the Kantorovich-Rubinstein duality applied to the Lipschitz loss function and then invoking the definition of sensitivity. We refer to Figure~\ref{fig:PR} for contrasting the performative confidence bounds with those of classical applications of Lipschitz bandits that follow from extrapolating $\PR(\theta)$ under a Lipschitzness assumption of the reward function.

\begin{figure}
    \centering
    \includegraphics[width=0.95\columnwidth]{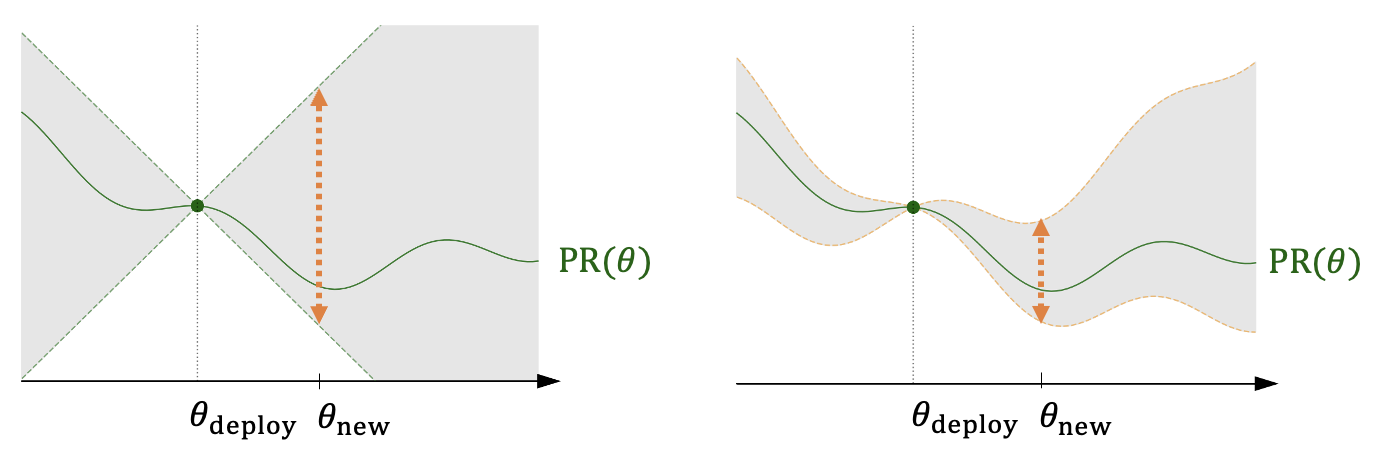}
    \caption{Confidence bounds on the performative risk. Left: using bandit feedback and Lipschitzness of the performative risk. Right: using performative feedback together with sensitivity of the distribution map and Lipschitzness of $\ell$ in $z$.}
    \label{fig:PR}
\end{figure}

From the figure it can been seen that performative confidence bounds potentially allow the learner to discard high risk regions of the parameter space without ever exploring them. 
Furthermore, from the bound in Equation~\ref{eq:ub} we see that the size of the confidence sets is independent of the dependence of the loss function on $\theta.$ This stands in contrast with a naive application of black box optimization techniques to $\PR(\theta)$. Furthermore, the extrapolation uncertainty diminishes as~$\epsilon\rightarrow 0$, recovering a static problem.

Building on the above construction, taking into account finite sample uncertainty and applying techniques from successive elimination~\citep{evendar02successiveEl} to determine which model to deploy next, \cite{jagadeesan22regret} proposed an algorithm that achieves performative regret of the form
\begin{equation}\mathrm{Reg}(T ) = \tilde O\left(\sqrt{T} +T^{\frac{d+1}{d+2}} (L_z\epsilon)^{\frac{d}{d+2}}\right)
\label{eq:regret}
\end{equation}
where $d$ denotes the zooming dimension of the problem; an instance-dependent notion of dimensionality introduced by \cite{kleinberg08bandit} that can be upper-bounded by the dimension of the parameter space for the purpose of this exposition.

In contrast to classical applications of Lipschitz bandits, the bound in Equation~\ref{eq:regret} scales as $\tilde O(\sqrt T)$ in the case of $\epsilon=0$ and only finite sample uncertainty remains. 

This result explains interesting particularities of the performative prediction problem for the application of the rich toolkit from  bandit optimization. An interesting future direction would be to borrow techniques from non-stationary bandits~\citep{bandit88} for dealing with exogenous distribution shifts co-occurring with performativity, or to apply ideas from contextual bandits~\citep{balseiro2018contextual} to deal with modifications to the learning objective. Another promising direction would be to investigate applications of constraint exploration as in~\cite{wu16conservativebandit, turchetta19safe}, to account for consequences of model deployment and safeguard against negative impact. See also \cite{grac15safe} for a starting point on related work in safe reinforcement learning.~\looseness=-1

\subsection{Model evaluation, algorithmic fairness and reward hacking}

Beyond predictive accuracy alone, we often also care about the specifics of how predictions affect a population. This is a concern central to the active research in the area of fairness in machine learning, see~\cite{barocas-hardt-narayanan}. 

Researchers have proposed numerous statistical \emph{fairness criteria} that formalize different notions of equality between different groups in the population, defined along the lines of categories such as \emph{race} or \emph{gender}. One criterion, for example, asks to equalize the true positive rate of the predictor in different demographic groups~\citep{cole1973bias, hardt2016equality}.
Such fairness criteria are typically thought of as constraints on a static supervised learning objective. Departing from this perspective,  \cite{liu2018delayed} proposed a notion of \emph{delayed impact} that models how decisions affect a population in the long run.  As an example, think of how lending practices based on a credit score affect welfare in the population. A default on a loan not only diminishes profit for the lender, it also impoverishes the individual who defaulted. The latter effect is typically not modeled in supervised learning. The results show that standard fairness criteria, when enforced as constraints on optimization, may not have a positive long-term impact on a disadvantaged group, relative to unconstrained optimization.

Delayed impact is an instance of performative prediction. It assesses the impact and the performance of a model \emph{after} deployment.  Moreover, it illustrates the importance of performativity as a criterion in the evaluation of machine learning. Adding further support to this point, \cite{fuster2022predictably} showed how the introduction of machine learning in lending can have a disparate effect on borrowers of different racial groups.  In a similar vein, \cite{hu2018short} studied the same problem in labor markets. \cite{ensign2018runaway} studied feedback loops in predictive policing. \cite{taori2023data} showed how performativity can amplify dataset biases. \cite{martin23predatory} showed how performativity coupled with economic incentives favors false positives in predictive analytics, termed predatory predictions. 
\cite{pan2024feedback} illustrate how performativity can incentivize reward hacking in generative models with negative externalities that do not surface in accuracy evaluations. 
Going beyond statistical evaluation criteria the celebrated work of \cite{coate1993will} is an important intellectual precursor of the development of dynamic models in algorithmic fairness. We turn to model-based approaches in performative prediction next.

\section{Model-based approaches to performative prediction}
\label{sec:model-based}

It can be appealing to model the distribution map for building a more precise understanding of the distribution shift and for incorporating structural assumptions and context-specific knowledge. This can serve more efficient optimization, but also to foresee consequences of future model deployments without the need to interact with the environment.
In this section we discuss several prominent models for the distribution map from the literature that all postulate a different mechanism behind performativity.

\subsection{Predictions as a mediator for performativity}
Recall the motivating example from Section~\ref{sec:GMStheorem}.
In many economic settings, a natural mediator for performativity are predictions. Meaning that two models that result in the same predictions share the same distribution map, irrespective of the model parameters that gave rise to these predictions. 
Incorporating this structural assumption into the performative prediction framework \cite{mofakhami23nn} proposed a sensitivity condition defined in prediction space, rather than parameter space. They showed how such a stronger notion requires weaker regularity assumptions on the loss for convergence of retraining, relaxing the convexity assumption with respect to the model parameters. This structural constraint is particularly useful for working with modern deep learning models characterized by high-dimensional parameter vectors.

Taking one step further, \cite{mendler22causal} constrain the performative effects, mediated by predictions, to only impact the outcome variable. Self-fulfilling and self-negating prophecies are instances of this model. The assumption is formalized using the following causal diagram for the data generating process:

\begin{figure}[h!]
\begin{center}
\begin{tikzpicture}[>=latex',circ/.style={draw, shape=circle, node distance=1cm, line width=1pt}]
    \node[circ] (x) at (0,0) {$X$};
    \node[circ] (y) at (2,0) {$Y$};
    \node[circ] (S) at (1,1.3) {$\hat Y$};
    \path[style={->,line width=1pt}] (x) edge (S);
    \path[style={->,line width=1pt}] (S) edge (y);
    \path[style={->,  line width=1pt}] (x) edge (y);
\end{tikzpicture}
\end{center}
\end{figure}

Let $X=\xi_X$ with $\xi_X\sim\cD_x$ denoting the feature vector that is drawn iid.~from a static feature distribution. The prediction is determined as $\hat Y =f_\theta(X)$ and the outcome is determined from $X$ and $\hat Y$ as $Y=g(X,\hat Y,\xi_Y)$ where $\xi_Y$ denotes an exogenous noise variable. In this model the structural equation $g$ encodes the mechanism of performativity. In words, we assume performative effects are mediated by the prediction $\hat Y$ and they impact the outcome variable $Y$, while features $X$ are unaffected by prediction. The model was later referred to as \emph{outcome performativity} by \cite{kim2023making}.

\paragraph*{Causal models, learning and evaluation.} Once a model of $\cD(\theta)$, such as the structural equation $g$ is known, this model can be used to anticipate the deployment of future candidate models $\theta\in\Theta$. It serves to evaluate the performative risk offline and compare models alongside different evaluation criteria and loss functions. Towards this goal, \cite{mendler22causal} focus on learning the structural equation $g$  from data collected under previous model deployments. More precisely, they treat the prediction as an additional input feature when learning a model for the outcome. This approach, termed \emph{predicting from predictions}, lifts the dynamic problem of performative prediction to a supervised learning problem of predicting $Y$ from $(X,\hat Y)$. The question whether such a data-driven approach can recover the true underlying causal mechanism~$g$ is a question of whether the training data permits causal identifiability~\citep{pearl09}. \cite{mendler22causal} emphasize the need to record predictions during data collection as a prerequisite to causal identifiability, and identified natural settings where causal identifiability can be established, given triplets $(X,\hat Y, Y)$. This includes scenarios where the data was collected under the deployment of a prediction function $f_\mathrm{collect}$ that was either appropriately randomized, had discrete outputs, or its functional complexity was sufficiently high relative to the true underlying concept. However, if $f_\mathrm{collect}$ is a continuous and deterministic prediction function, these guarantees cease to hold. 
Thus, causal identifiability of performative effects poses an important obstacle in practice when learning from observational data.~\looseness=-1

Assuming access to data collected under a randomized predictor, \cite{kim2023making} studied the task of learning the structural equation $g$ through the lens of outcome indistinguishability~\citep{dwork21outcomeind}; a weaker condition compared to causal identification that is specific to a down-stream learning goal. A model for the structural equation $g$ is termed a \emph{ performative omnipredictor} with respect to a class of loss functions $\mathcal L$ and a hypothesis class $\cF$ if it is sufficient to identify the performative optimal model within a model family $\mathcal F$ for any $\ell\in\mathcal L$. The optimality concept of an omnipredictor is adapted from the supervised learning setting where it was introduced by \cite{gopalan21omnipredictor}.
\cite{kim2023making} showed that, irrespective of the complexity of $g$, for sufficiently simple classes $\mathcal L$ and $\mathcal H$ there exists a simple omnipredictor. 

\cite{perdomo2023difficult} apply the approach of \cite{mendler22causal} to test for performativity in a major evaluation of an early warning system for high school dropout used across the state of Wisconsin. If the early warning system were effective, predictions of dropout risk should be predictive of the actual outcome and have an effect that is self-negating.

The development of more versatile methods to test for performativity beyond asking for causal identification is an interesting direction for future work. As an avenue forward, \cite{cheng2023causal} draw on connections to control theory for estimating performativity from repeated interactions of users with a platform.

\subsection{Strategic classification and microfoundations}
\label{sec:strat-class}

An alternate set of assumptions underlies the model of strategic classification~\citep{bruckner2011stackelberg, hardt16strategic}.  In this line of work, the response of a population to a predictor consists of strategic feature modifications by individuals. These modifications derive from standard economic assumptions of a rational, representative agent, maximizing their utility function. These assumption sometimes go by the name of \emph{microfoundations}. The literature on this topic is vast and interdisciplinary,  see \cite{janssen2005microfoundations} for a starting point. The general idea of microfoundations is to derive aggregate response functions from microeconomic principles of utility theory and rational behavior.~\looseness=-1 

Concretely, in the setting of strategic classification individuals  modify their features in response to a predictor with the goal of achieving a better outcome. We assume a perfectly rational and utility maximizing agent. We further assume that all individuals strategize in the same way, which justifies considering one representative agent. The firm leads by deploying a predictor. The agent follows by changing their features with full knowledge of the deployed classifier. This sequential two player game is an instance of a Stackelberg game. The agent's action~$x(\theta)$ in response to a model~$\theta$ satisfies
\begin{equation}
x(\theta) = \argmax_{x'\in\cX} f_\theta(x')-c(x_\mathrm{orig},x')\,
\label{eq:sc}
\end{equation}
where $c(x_\mathrm{orig},x')$ denotes the cost of moving from the original features $x_\mathrm{orig}$ to the new features $x'$. Thus, the agent engages in feature manipulation so long as the benefit of a positive prediction exceeds the cost of feature manipulation. The distribution map is then determined by aggregating the agent behavior across the population.

This model can capture different types of strategic behavior around machine learning applications. Often described as \emph{gaming the classifier}, the strategic behavior can also represent attempts at self-improvement \citep{miller20strategic, kleinberg2020classifiers}. We refer to \cite{rosenfeld24strattutorial} for a comprehensive overview of recent results related to strategic classification.

\paragraph*{Micro-macro tension.}
Due to the precise characterization of the agent's response function, the strategic classification model lends itself to mathematical analysis. However, the standard microfoundations of a rational representative agent are the subject of well-known critiques in macroeconomics \citep{kirman1992whom, stiglitz2018where} and sociology \citep{collins1981microfoundations}. Similarly, \cite{jag21alt} pointed out important drawbacks of these assumptions in the context of performative prediction. In particular, standard microfoundations lead to the following problems:
\begin{itemize}
\item 
There are sharp discontinuities in the response to a decision rule that are often not descriptive of the empirical reality.
\item 
While retraining converges to stable points under natural conditions on the cost function, if only a positive fraction of agents act non strategically stable points cease to exist.
\item
Performatively optimal decision rules derived from standard microfoundations maximize a measure of negative externality known as social burden~\citep{milli2019social} within a broad class of alternative assumptions about agent behavior.
\end{itemize}

\cite{jag21alt} propose a noisy response model, inspired by smoothed analysis~\citep{spielmann09smoothedanalysis}, that leads to more robust conclusions about performative stability and less conservative performative optima. Finding the right interface between micro modeling and expected risk minimization as a macro concept, remains a major open challenge.

\subsection{Parametric models}

Taking a more abstract approach, \cite{miller2021outside} study distribution maps $\cD(\theta)$ that can be expressed in parametric form. In particular, they study location-scale families where samples $z\sim\cD(\theta)$ can be expressed as $z=z_0+\mu\theta$ with $z_0\sim\cD_0$ being a sub-Gaussian random variable and~$\mu$ denoting the location map. Such a parametric model is particularly appealing as it leads to desirable convexity properties of the performative risk.
Similar in spirit, \cite{izzo21gd} use an exponential family to model the distribution map.  They model the distribution map as a mixture of $K$ normal distributions as $\cD(\theta)=\sum_{i=1}^k \alpha_i \mathcal N(\mu_i(\theta), \Sigma_i)$ where means depend on the deployed model and covariances are fixed, and $\alpha_i>0$ are positive coefficients summing to one.

\paragraph*{Model misspecification.} \cite{lin2023plugin} characterized regimes where parameterized models, even if misspecified, can be beneficial for learning under performativity. In particular, they show that in the regime of finite samples, model-based approaches can find near performative optimal points more efficiently compared to structure-agnostic approaches such as discussed in Section~\ref{sec:po}. To provide intuition for this claim, let $\cD_\beta(\theta)$ be a parameterized model for the distribution map $\cD(\theta)$, and denote $\hat\theta_{\mathrm{PO},\beta}$ the performative optimum computed under this model. Let $\hat \beta$ be a finite sample estimate of $\beta$. Then, the excess risk of a model-based approach decomposes into a model misspecification term and a statistical error term, and it is bounded by
\[\PR(\hat\theta_{\mathrm{PO},{\hat\beta}})-\PR(\thetaPO)\leq \mathrm{MisspecErr} + \tilde O\left(\frac 1 {\sqrt{n}}\right)\]
While the misspecification error is irreducible, the statistical error vanishes at a rate $\tilde O(n^{-\frac 1 2})$. In contrast, the bandit algorithm discussed in Section~\ref{sec:bandit} does not suffer from model misspecification error, but has an exceedingly slow statistical rate of $\tilde O(n^{\frac 1 {1+d}})$. Thus, if the misspecification error is not too large, and data is scarce, model-based approaches to performative prediction can outperform model-free methods.

\section{Power and the strength of performativity}
\label{sec:power}

Morgenstern already knew that performativity was a consequence of authority and reach. He recognized that the strength of performativity varied from the case of a monopoly forecast to the case of multiple competing forecasts \cite[p.111]{morgenstern28}.
Put differently, the extent to which predictions influence the course
of events depends on economic conditions and  \emph{the power} of those who make the prediction. The traffic predictions of a service used by millions of drivers have the ability to significantly shape traffic. Were a competing service to enter the market, its predictions would affect a smaller user base and hence show a lesser effect on traffic patterns. This in turn implies that the optimization problem, and the best prediction for the two scenarios can differ widely. As we discussed earlier, performativity rewrites the rules of prediction, and these rules depend on the strength of performativity.

The relationship between economic power and performativity is the starting point for the definition of \emph{performative power} by \cite{hardt2022power} that we turn to next.


\subsection{Performative power}
Performative power measures the strength of performativity. In other words, it quantifies the ability of a platform to steer a population of participants. Returning to our example of traffic prediction, performative power is concerned with how much a platform can influence traffic patterns through changes to its published predictions.

To formally set up the definition, fix a set~$\cU$ of participants interacting with a firm under investigation. Each unit~$u\in\cU$ is associated with a data point~$z(u)$. Fix a metric $\distance(z,z')$ over the space of data points. Fix a set of actions $\cF$ for the firm you want to consider. We typically think of an action~$f\in\cF$ as the deployment of a specific algorithmic change at a fixed point in time. For each participant~$u\in\cU$ and action~$f \in\cF$, we denote by $z_f(u)$ the potential outcome random variable representing the counterfactual data of participant~$u$ if the firm were to take action~$f$.

\begin{definition}[Performative Power]
\label{def:performativepower}
Given a population $\cU$, an action set $\cF$, potential outcome pairs $(z(u),z_f(u))$ for each unit~$u\in\cU$ and action~$f\in\cF$, and a metric $\distance$ over the space of data points, we define the 
\emph{performative power} of the firm as
\[
\PP:=\sup_{f \in \cF}\; \frac 1 {|\cU|} \sum_{u \in \cU} \E\left[\distance\left(z(u),z_f(u)\right)\right],
\]
where the expectation is over the randomness in the potential outcomes.
\end{definition}
%



Digital platforms typically operate in two-sided markets mediating between content creators or service providers and consumers. The definition applies to both sides of the market. We could study how much a change in a ranking algorithm influences the wellbeing of video producers. Or we could study how much the positioning of search results steers visitors toward certain product offerings. The definition is flexible in quantifying diverse power relationships. The action and the outcome variable do not even need to concern the same market. Performative power can be applied forward-looking to understand whether a platform has the ability to plausibly cause a specific change, as well as in retrospect to quantify the effect of an observed conduct.

Before we turn to the question of empirical estimation and economic relevance, we collect some useful theoretical properties that the definition enjoys.

\subsection{Theoretical properties} 

Performative power can be linked back to performative prediction. Given the action set corresponds to the entire parameter space $\cF=\Theta$, and the outcome variable to the observed data point $z=(x,y)$, then the benefit of steering is bounded by performative power. More precisely,  under appropriate regularity conditions, any supervised learning solution $\theta^*$ with respect to $\cD(\phi)$ for $\phi$ being some previously deployed model, satisfies
\[\PR(\theta^*)-\PR(\thetaPO)< \xi \cdot \PP\,.\]
where $\xi\geq 0$ depends on the Lipschitz constant of the loss.
This inequality illustrates how performative power bounds the extent to which a firm can, in principle, utilize the mechanism of steering to its own advantage. However, note that it is an upper bound that is not necessarily tight for every loss function. But what it tells us is that a firm with zero performative power is confined to optimizing its objective against current conditions. A firm with positive performative power, however, can gain from changing conditions. This dichotomy is analogous to the distinction between a \emph{price taker} and a \emph{price maker} in economics.

\paragraph*{Performative power and competition.}
\cite{hardt2022power} further investigated the connections between performative power and standard economic concepts from competition. To this end they study a simple microfoundation model similar to the one we saw in Section~\ref{sec:strat-class}, extending it to a multi-decision maker setting in a straight forward way. They found:

\begin{enumerate}
\item A monopoly firm maximizes performative power. Participants are willing to incur a cost for feature manipulation equal to utility for the service. At the same time, outside options decrease performative power.
\item A firm’s ability to personalize predictions increases performative power.
\item When multiple platforms compete over users with services that are perfect substitutes, then already two firms lead to zero performative power. This result is analogous to the Bertrand competition \citep{bertrand1883theorie} where two competing firms setting prices for identical goods simultaneously are sufficient for zero profit, i.e., competitive prices.
\end{enumerate}

These propositions suggest that the definition is appropriately sensitive to some key concepts from the study of competition. The mode of investigation here is to see how the definition squares against familiar economic concepts, at the example of social burden as a metric of potential negative externalities in strategic classification. An interesting avenue for future work is to deepen the study of performative power in the context of established economic models. For example, it may be insightful to understand how performative power behaves under market entry, mergers and acquisitions, and common ownership.

\subsection{Practical advantages}
Emerging from performative prediction, performative power is a notion of power tailored to digital platforms. As an investigative tool it has the potential to illuminate settings where traditional concepts of market power are difficult to apply. There is extensive discussion about the difficulty of applying classical tools to digital economies \citep{cremer2019competition, stigler19}. An important hurdle being behavioral biases and weaknesses, as well as the reliance on prices and mathematical specifications for the market. Against this backdrop, performative power as a causal statistical notion offers a promising alternative. The concept can be applied to any outcome variable of interest, beyond prices, without presupposing a precise understanding of the complex mechanism behind the market in which the firm operates. While the dynamic process that generates the potential outcome~$z_f(u)$ may be complex, this complexity does not enter the definition. 
At the same time, performative power offers a concrete path towards integrating online experiments with digital market investigations. However, the legal relevance of the notion depends on its instantiation.

\subsection{Measuring performative power via causal inference}
An important criteria for a useful definition is the feasibility to instrumentalize it. Having instantiated the sets~$\cF$ and~$\cU$, estimating performative power amounts to causal inference, involving the potential outcome variables~$z_f(u)$ for unit $u\in\cU$ and action $f\in\cF$. The definition takes a supremum over possible actions a firm can take at a specific point in time. We can therefore lower bound performative power by estimating the causal effect of any given action $f\in\cF.$ 
In an experimental design, the investigator chooses an action and deploys it to estimate the effect. In an observational design, an investigator is able to identify the effect and estimate performative power from available data without an experimental intervention on the platform. 

\paragraph*{Content positioning in digital services.}
To illustrate the definition more concretely, we consider a digital platform that uses a ranking algorithm to place content in one of several positions. Through this example we outline a general strategy to derive a lower bound on performative power. 
Formally, we assume that there are $c$ pieces of content $\mathcal{C} = \{0, 1, 2,\dots, c-1\}$ that the platform can present in $m$ display slots. We focus on the case of two display slots ($m=2$) since it already encapsulates the main idea. The first display slot is more desirable as it is more likely to catch the viewer's attention. This is quantified through the causal effect of position.~\looseness=-1

\begin{definition}[Causal effect of position]
Let the treatment $T\in\{0,1\}$ be the action of flipping the content in the first and second display slots for a viewer $u$, and let the potential outcome variable~$Y_t(u)$ indicate whether, under the treatment~$T=t$, viewer~$u$ consumes the content that is initially in the first display slot. We call the corresponding average treatment effect
\[\beta = \Big|\frac{1}{|\cU|}\sum_{u\in\cU}\mathbb{E}\left[Y_1(u) - Y_0(u)\right]\Big|\]
the \textit{causal effect of position} in a population of viewers $\cU$, where the expectation is taken over 
the randomness in the potential outcomes. 
\label{def:display-effect}
\end{definition}

Researchers have investigated the causal effect of position on consumption in various contexts; often via quasi-experimental methods such as regression discontinuity designs, but also through experimentation in the form of A/B tests. 
To give an example, \citet{NK15} used a regression discontinuity design to estimate the causal effect of position in search advertising, where advertisements are displayed across a number of ordered slots whenever a keyword is searched. They measured the causal effect of position on click-through rate of participants. \cite{mendler2024engine} designed a randomized online experiment based on a browser extension to estimate the effect of content positioning on clicks in the context of online search. 

The following theorem relates the causal effect of position to performative power. It assumes that firm actions are exposed to users through content arrangements. Given a scoring function, the firm ranks and assigns content items to display slots. Thus, given a set of actions containing local perturbations to the score function that exceed the score value gap between two adjacent elements, the causal effect of position lower bounds performative power provided that units don't interfere. More concretely, let $\cU$ be the population of platform users. Let $z(u)$ be the probability that user $u\in\cU$ clicks on item $1$. Let the set of actions compose of local perturbations to the scoring function as described above. Then, the following  holds.

\begin{theorem}[\cite{hardt2022power}]
\label{thm:swapping}
Let $\PP$ be performative power as instantiated above. Assume no interference between experimental units. Then, performative power is at least as large as the causal effect of position $\beta\le \PP\,.$
\end{theorem}

The proof idea is to show that the action set $\cF$ contains at least one action that results in flipping the position from the first to the second content item. Then, to show that performative power is large, we need to apply the firm's action simultaneously to all units in the population. If units were to interfere, then it would not be clear that the causal effect extrapolates from a single unit to the whole population. This is what the no interference assumption takes care of. In the context of search advertising, this assumption concretely means that one user's clicks don't influence another user's clicks. This is expected to hold within a short time window as users generally do not observe each other's clicks.

The same argument underlying Theorem~\ref{thm:swapping} of first choosing one action to lower bound the supremum and then relating this update to a display change could be applied to any other $f\in\cF$ beyond flipping the order of two adjacent positions. Depending on the use case alternative choices of $f$ could lead to a tighter lower bound. We discuss one use case below.

\subsection{Integrating performative power with digital market investigations}

Questions about how much a platform can influence participants through its algorithmic actions are at the heart of several recent antitrust investigations into digital platforms. Prominently, cases of self-preferencing allege that platforms used positioning strategically to steer users toward their own product offerings~\citep{cremer2019competition, stigler19}. 
To quote from a major antitrust case by the European Commission against Google:

\begin{quote}
[T]he General Court [of the European Union] finds that, by favouring its own comparison shopping service on its general results pages through more
favourable display and positioning, while relegating the results from competing comparison
services in those pages by means of ranking algorithms, Google departed from competition
on the merits.\footnote{General Court of the European Union. Press Release No 197/21.}
\end{quote}

\noindent
Similarly, a recent complaint by the Federal Trade Commission\footnote{Federal Trade Commission. Case 2:23-cv-01495} against Amazon includes among its allegations the claim that

\begin{quote}
    Amazon deliberately steers shoppers away from offers not featured in the Buy Box. 
\end{quote}

\noindent
Note the invocation of the notion of \emph{steering} to describe Amazon's practices. Performative power gives a rigorous and effective set of tools to support and quantify precisely such claims.

To illustrate the role of performative power, take the Google Shopping case\footnote{European Commission, AT.39740, Google Search (Shopping), 27.06.2017.} as an example. It is concerned with the ability of Google to distort competition in the market of comparison shopping services by changed to its search service. 
Accordingly, the relevant instantiation of performative power spans two vertically integrated markets. It quantifies Google's ability to steer \emph{incoming} traffic to a comparison shopping service by means of algorithmic actions that determine where the service appears on Google search relative to its competitors. 
Theorem~\ref{thm:swapping} speaks to performative power as the ability to steer the behavior of participants on a platform through content arrangement. To factor in that only a subset of the relevant web traffic is mediated by Google search, we can build on the decomposition property of performative power. Namely, under the conditions of Theorem~\ref{thm:swapping}, for any pair of populations $\cU,\cU'$ with $\cU'\subseteq \cU$, it holds that $\PP(\cU)\geq\alpha \PP(\cU')$ with $\alpha=|\cU'|/|\cU|$.
Thus, for extending empirical insights from on-platform behavior of participants to broader claims about power it remains to quantify the portions of traffic the measurements concern. Such information can typically be obtained from public web traffic data. Let us make this more concrete using the following thought experiment sketched by \cite{mendler2024engine}:

\textit{Through online experiments the authors found that distortion of traffic at the first position can be as large as $66\%$ through consistent down ranking. Suppose, $80\%$ of the search referrals to a site come from Google Search.\footnote{Up to $82\%$ of incoming traffic to CS services in the European Economic Area (EEA) was mediated by Google search in 2012,  reported in the Google Shopping case decision (Section 7.2.4.1, Table 24).} Further, assume that $70\%$ of the referrals from Google happen while the service is ranked among the top two generic search results. Assuming for the second position the traffic distortion is $20\%$ smaller compared to the first position, giving a conservative average effect size of $\beta=0.8\cdot 0.66$. Multiplying the effect size by the fraction of incoming clicks it concerns, we get $0.8\cdot 0.7\cdot 0.8\cdot 0.66\approx 30\%$. This is the fraction of traffic to the competitor's site Google can influence. }

This number provides a conservative lower bound on performative power of online search for a legally and economically relevant instantiation. Furthermore, it offers an interpretable measure for an investigator to judge whether a firm's power to steer traffic should be a concern for competition or not. The calculation also illustrates that for performative power to be large neither control over local traffic, nor a large market share alone are sufficient but both factors need to be large. We can use the same logic to compare search engines, and assess the effectiveness of remedies.  
More broadly, performative power offers an avenue to integrate online experimentation and insights from performative prediction with ongoing legal debates related to the power of digital platforms.

\section{Algorithmic collective action}
\label{sec:collective-action}

In the previous section we have discussed how much a firm can steer a population through its choice of predictions, and linked it to the benefits for the firm. In this section we investigate how participants can use the knowledge that platforms perform statistical learning to their advantage and ask the questions: How can a fraction of the population adapt their behavior to steer the predictions of a firm towards desirable outcomes?

\cite{vincent21datalever} studied this question from the perspective of data leverage as a way to reduce the effectiveness of lucrative technologies. Participants on a platform strategically modify their data so as to influence the behavior of the predictor learned by the platform. The more individuals participate in the collective, the larger their leverage. One motivation for such collective strategies is to improve the conditions of gig workers on digital platforms \citep{chen2018thrown, wood2019good, gray2019ghost, schor2020dependence, schor2021after}. For example, recent work in machine learning has focused on strategies to increase underrepresented artists' exposure on music streaming platforms~\citep{baumann2024algorithmic} or raise prices on ride hailing platforms~\citep{sigg2024decline}.

\paragraph*{Mathematical model.}
\cite{hardt2023algorithmic} proposed a theoretical model to study algorithmic collective action in machine learning more generally. The size of the collective is specified by a value~$\alpha\in (0,1]$ that corresponds to the fraction of participating individuals in a population drawn from a base distribution~$\cP_0.$ The firm observes the mixture distribution
\[\cP=\alpha \cP^* + (1-\alpha) \cP_0\,,\]
and runs a learning algorithm $\cA$ on $\cP$, where $\cP^*$ depends on the strategy of the collective. We illustrate the case of an optimal firm that has full knowledge of the distribution $\cP$. The firm chooses the Bayes optimal predictor~$f=\cA(\cP)$ on the distribution~$\cP$.

\subsection{Effectiveness of signal planting strategy}

In the context of classification, a natural objective for the collective is to either induce or erase correlations between features and labels in the training data. We focus on the former. The collective correlates a signal $g(x)$ applied to a data point $x$ with a target label $y^*.$  To do so, given a data point $(x,y),$ the collective applies the signal function $g(\cdot)$ to $x$ and switches the label from $y$ to $y^*$ at training time. That is, $\cP^*$ is the distribution of $(g(x), y^*)$ for a random draw of a labeled data point~$(x, y)$ from~$\cP_0$. In practice, the signal~$g(x)$ may correspond to adding a hidden watermark in image and video content, or subtle syntactic changes in text. It is reasonable to assume that individuals are indifferent to such inconsequential changes. Conventional wisdom in machine learning has it that such hidden signals are easy to come by in practice~\citep{liu17trojan,chen2017targeted, gu19badnet}. 

At test time, an individual given a data point $(x,y)$ succeeds if $f(g(x))=y^*$. We denote the success probability of the strategy by
\[
S(\alpha) = \Pr_{x\sim \cP_0} \left\{ f(g(x)) = y^*\right\}\,.
\]
A theorem shows that a collective of vanishing fractional size succeeds with high probability by implementing this strategy, provided that the signal $g(x)$ is unlikely to be encountered in the base distribution~$\cP_0$. We denote the density of the signal set under the base distribution by $\xi=\cP_0(\{ g(x) \colon x\in\cX\}).$

\begin{theorem}[\cite{hardt2023algorithmic}]
\label{thm:trigger-label}
Under the above assumptions, the success probability of the collective is lower bounded as:
\[
S(\alpha)\ge 1 - \frac{1-\alpha}{\alpha} \cdot \xi  \,.
\]
\end{theorem}

The theorem has various extensions. The firm need not be optimal. The collective need not be able to change labels. Up to some quantitative loss, the result still holds.

The theorem shows how the collective can influence predictions, how this depends on the size of the collective, and what determines the effectiveness of the strategy. However, this result does not speak about actual outcomes, such as monetary rewards for the collective. Surely, the collective is also interested in the latter and not just the former. Performativity provides the necessary link between predictions and outcomes to talk about incentives and outcomes of the collective. We present this essential idea here.

\subsection{Algorithmic collective action under performativity}

Consider an outcome variable 
\[
Y = h(X) + \beta f(X) + Z\,,\]
where $Z$ is exogenous noise (i.e., independent of $X$) of mean~$0$. Think of $Y$ as the revenue earned by an offering~$X$ on a platform. Here, $f(X)$ is a binary predictor based on which the platform decides whether to promote the offering~$X$ in some form. The term $h(X)$ captures the baseline revenue that offering $X$ earns absent the platform intervention. The factor $\beta$ quantifies the strength of performativity. 

Assume the collective succeeded in finding a signal function~$g$ such that the baseline revenue~$h$ is invariant under application of the signal function, i.e., $h(g(x))=h(x)$ for all $x\in\cX.$
Let~$C$ indicate the event that the drawn data point belongs to an individual participating in the collective. The expected revenue is $\E[h(X)],$ whereas the realized revenue for the collective is $\E[Y\mid C]$. It is not difficult to show that
\begin{equation}\label{eq:revenue}
\E\left[Y \mid C\right]
- \E\left[h(X)\right] \ge  S(\alpha) \beta  \,.
\end{equation}
Assuming a rare enough signal with $\xi\le\alpha/2,$ Theorem~\ref{thm:trigger-label}
shows that the revenue increase for the collective is lower bounded by $\beta/2.$ In other words, the strength of performativity determines the realized payoff for the collective. To prove Equation~\ref{eq:revenue}, note:
\begin{align*}
    \E\left[Y \mid C\right]
    &= \E\left[h(g(X)) + \beta f(g(X)) + Z \mid C\right] \\
    &= \E\left[h(g(X))\right] + \beta \E\left[f(g(X))\right] + \E\left[ Z \mid C\right] \\
    &= \E\left[h(X)\right] + \beta \Pr\left\{f(g(X))=1 \right\}\\
    &= \E\left[h(X)\right] + S(\alpha) \beta
\end{align*}
This simple observation gives a first indication that there may be a monetary incentive to algorithmic collective action. 
Incentives and dynamics of collective action beyond machine learning have been studied extensively since \cite{olson1965logic}. It is an interesting direction for future work to study these problems in the context of predictive systems,  go beyond adversarial attack models, and ground them in economic interests. Performativity offers an important building block towards better understanding incentives in algorithmic systems, taking into account the dynamic interaction of individual behavior, predictions and outcomes.

\section{Discussion}

Performativity is a ubiquitous phenomenon in the context of prediction. We can observe it in many applications of machine learning. Predictions on digital platforms, be it for online recommendation, content moderation, or digital advertising, are one particularly rich class of examples. But we would go so
far as to conjecture that any significant social prediction problem has some
aspect of performativity to it. 

While performativity is ubiquitous, the strength of performativity has an important quantifier. It depends on the power, in terms of scale, reach, and visibility, of the predictor. This property of performativity has far-reaching consequences. If power is the cause of performativity, we can measure power through the strength of performativity. This principle forms the basis of a formal notion of power in the context of prediction.

Performativity is also a notion that we can use to talk about what machine learning systems do and ought to be doing. As such, performativity becomes a criterion in the design and evaluation of machine learning systems in social contexts.  In the same way that we think about machine learning systems in terms of other criteria, such as accuracy or computational efficiency, we propose adding performativity to the principal axes along which we scrutinize predictive systems. 

When predictions have consequences, the predictor becomes a point of intervention. This cuts two ways. On the one hand, a firm or an institution can use predictions toward social and economic objectives they pursue. On the other hand, individuals can leverage their data to influence social outcomes through the predictor trained on the data.

Performative prediction is a framework to study performativity as a phenomenon and to formalize the notion within the context of supervised learning. The framework establishes the necessary formal language and definitions for its dual use in the study of performativity as a phenomenon and a criterion. Compared with traditional supervised learning, there are two ways to be good at performative prediction. One is the familiar mechanism to learn patterns on current data. The other is to \emph{steer} the data-generating process in a different direction.  
A firm without performative power only has the first lever, whereas a firm with high performative power may benefit from steering the population towards more profitable behavior. This is a distinction classical supervised learning can't make.

Our conception of performative prediction excludes some important interdisciplinary perspectives. The celebrated work of Mackenzie \citep{mackenzie2007economists, mackenzie2008engine}, for example, speaks to the performativity of economic theory. Here it is the practice of applying economic theory in financial markets and not the model's predictions per se that are performative. 
Another notable form of performativity that is not captured by our framework is Hacking's looping effect. \cite{hacking1995looping} describes how classification creates and changes social categories in an ongoing loop. This effect arises in machine learning applications as well, but is not adequately captured in our formalism. There is a vast literature on performativity in philosophy and sociology starting with the works of \cite{austin1975things}, \cite{merton1948self}, and \cite{buck1963reflexive} who called the problem \emph{reflexive prediction}, that we cannot survey here. See \cite{mackinnon2005reflexive, maki2013performativity, khosrowi2023managing} for additional pointers and discussion.

Although the definition of performative prediction is recent, work in this area already spans a remarkable scope.  Questions of strategic behavior, algorithmic fairness, prediction as intervention, power in digital economies, and collective action in algorithmic systems, to name a few, all benefit from the perspective of performative prediction. The problems that arise are technically rich and thrive on beautiful connections with optimization, statistics, economics, causal inference, and control theory. We hope that this article provided a helpful entry point that has enticed the reader to this emerging area.

\newpage

%% file: introduction.tex
Long before his work with Von Neumann that founded the field of game theory, economist Oskar Morgenstern studied what he called one of the most difficult and most central problems in prediction. Emboldened by the advances of statistics in the 1920s, many of Morgenstern's contemporaries were eager to apply the new statistical machinery to the problem of charting the course of the economy. Morgenstern believed that
this was a fool's errand. Economic forecasting, he argued in his century-old habiliation, was impossible with the tools of economic theory and statistics alone~\citep[p.~112]{morgenstern28}.

Morgenstern had identified a compelling reason for his pessimistic outlook on prediction. Any economic forecast, published with authority and reach, would necessarily cause economic activity that would influence the outcomes that the forecast aimed to predict. This causal relationship between a prediction and its target, Morgenstern held, necessarily invalidated economic forecasts. In his argument, Morgenstern emphasized the difference between forecasting natural events and forecasting social events. He believed the problem that clouded economic forecasts was fundamental to predictions about the social world at large. 

We call the phenomenon Morgenstern so accurately described
\emph{performativity}. It refers to a causal influence that predictions have on the target of prediction. The empirical reality of this phenomenon is not limited to economic forecasts. The predictions of a content recommendation model on a digital platform are another good example. If the model predicts that a visitor will like a video and thus displays it prominently, the visitor is more likely to click and watch the video. Content recommendations, therefore, can be self-fulfilling prophecies. Traffic predictions, on the other hand, can be self-negating.  If
the service predicts that traffic is low on a certain route, drivers will switch over and increase traffic.

Does Morgenstern's argument doom prediction in the social world to guesswork with unforeseeable consequences? 

An attempt at a formal counterpoint to Morgenstern's argument came thirty years later in a paper by Emile Grunberg and Franco Modigliani, and in a contemporaneous work by Herbert Simon. \cite{grunberg54} studied prices, whereas \cite{simon54} turned to bandwagon and underdog
effects in election forecasts. The three scholars argued that it is in
principle possible to find a prediction that equals the outcome caused by the
prediction. All that is needed is the continuity of the function that relates
predictions to outcomes.

The work of Grunberg and Modigliani marked the start of a revolution in economic
theory that some hoped would solve the issue of performativity altogether \citep[p.~316]{muth1961rational}; \citep[pp.~51--52]{sent1998evolving}. But the
economic theorizing around performativity was lost on the development of
statistics and machine learning. Until recently, the theoretical foundations
of statistics and machine learning excluded the kind of feedback loop between model and data that characterizes performativity. The predominant risk minimization formulation of machine learning assumes an immutable data-generating process impervious to any model's predictions.

\subsection{Contributions and outline}

This article provides an orientation to help readers access and navigate the emerging area of performative prediction founded by \cite{PZMH20}. At the outset, performative prediction allows a chosen model to have an influence on the data-generating process. It retains all other aspects of the familiar risk formulation of supervised learning. Our goal is to put the many individual contributions to the area of performative prediction into a broader unified perspective, synthesizing material from different sources. 

We motivate the technical sections with a simple exposition of the results by Grunberg, Modigliani, and Simon from the 1950s. Section~\ref{sec:performative-prediction} presents the main framework of performative prediction. Sections~\ref{sec:retraining}--\ref{sec:model-based} give a simplified exposition of key optimization results in the area, distinguishing between model-free and model-based results. Starting with Section~\ref{sec:power} we consider the fundamental role that power plays in the context of performative prediction. We discuss the notion of \emph{performative power} and its potential to inform ongoing antitrust investigations. Section~\ref{sec:collective-action} presents the recently studied problem of algorithmic collective action, 
where we add a novel result connecting performative power and collective action.

\paragraph*{Acknowledgments.}
The primary content of this survey is based on several joint works with Meena Jagadeesan, Juan Carlos Perdomo, and Tijana Zrnic.